\documentclass[journal]{IEEEtran}

\usepackage{mathtools,cite}
\usepackage{graphicx}
\usepackage[caption=false,font=footnotesize]{subfig}
\usepackage{amsmath}
\usepackage{amssymb}
\usepackage{amsfonts}
\usepackage{amsthm}
\usepackage{bm}
\usepackage{algorithm}
\usepackage{algorithmic}
\usepackage{booktabs}
\usepackage{array}
\usepackage{xurl}
\usepackage{microtype}
\usepackage{hyperref}

\newtheorem{assumption}{Assumption}
\newtheorem{theorem}{Theorem}
\newtheorem{corollary}{Corollary}

\newtheorem{assertionprocess}{Assertion process}

\newcommand{\vect}[1]{\boldsymbol{#1}} 

\newcommand{\Prob}{\mathbb{P}}

\newcommand{\one}{\mathbf{1}}

\newcommand{\calH}{\mathcal{H}}

\newcommand{\calW}{\mathcal{W}}
\newcommand{\BetaInv}{\mathrm{B}^{-1}}
\newcommand{\Bin}{\mathrm{Binomial}}

\begin{document}

\title{Finite-Sample Probabilistic Safety Certification for AI-Based Grid-Edge Coordination}

\author{Yihong Zhou,~\IEEEmembership{Member,~IEEE}, Hanbin Yang, and Thomas Morstyn,~\IEEEmembership{Senior Member,~IEEE}%
\vspace{-3mm}
\thanks{Department of Engineering Science, University of Oxford, U.K.}%
\thanks{Corresponding author: Yihong Zhou (e-mail: yihong.zhou@eng.ox.ac.uk).}}

\markboth{Submitted to IEEE Transactions on Smart Grid}{}

\maketitle

\begin{abstract}
Coordinating a large population of flexible grid-edge devices can alleviate the need for time-consuming and capital-intensive network upgrades. Artificial intelligence (AI)-based control methods such as multi-agent reinforcement learning or imitation learning are promising to this coordination due to their real-time decision scalability. However, system operators still need an independent and rigorous way to decide whether a given AI system is safe enough for deployment. This paper develops a \emph{finite-sample} probabilistic safety certification framework for AI-based grid-edge coordination. The central idea is to reduce the complete input--AI--decision--evaluation workflow to a binary unsafe outcome under an operator-defined safety specification, and then use exact binomial inference to certify the corresponding unsafe operation probability. Given a set of held-out calibration scenarios, the framework returns the tightest one-sided upper certificate and an accept/reject deployment criterion that controls the probability of false safety certification. Because the certification is for the calibration distribution that may \emph{deviate} from the future operation, we further combine the nominal certificate with physically interpretable sample-space adversarial attacks, a concept widely used in AI models to investigate their fragility. Case studies on grid-edge flexibility coordination with 1{,}000-agent AI models (independent parameters) verify the safety guarantee and the value of integrating adversarial attacks into a rolling-window training-certification-deployment flow.
\end{abstract}


\section{Introduction}
\label{sec:introduction}

\IEEEPARstart{P}{ower} grids are facing increasing stress from variable renewable generation and millions of electrified grid-edge devices such as electric vehicles, heat pumps, and distributed batteries. Smart coordination strategies that exploit the flexibility of these grid-edge resources can defer capital-intensive network upgrades, with a global value estimated at USD~\$270~billion by 2040 \cite{iea_value}. Because of the high complexity and scale, such coordination problems motivate artificial intelligence (AI)-based solution methods. Existing studies have demonstrated the performance of AI-based control such as multi-agent reinforcement learning and imitation learning \cite{qiu2023reinforcement, charbonnier2025centralised}. Recent work has also pushed the scale to 1{,}000 independent AI agents coupled through non-convex alternating current (AC) distribution network constraints \cite{zhou2026gradmap}. These works show a promising picture that millions of flexible grid-edge devices could be coordinated by AI-based control systems, helping to unlock substantial value in terms of lower operating cost, increased reliability, and lower pollution.

The main challenge is that the black-box nature and limited explainability of AI systems can prevent system operators from deploying such models with confidence \cite{horizon2024_tef_energy, su2025review}. Existing studies that address this issue can be broadly categorised into safe learning and formal verification. Safe learning investigates approaches to \emph{improve} the safety of an AI system by modifying the learning or decision process, where a popular sub-category is safe reinforcement learning (RL) \cite{su2025review}. This includes primal-dual learning that continuously increases the violation penalty \cite{park2024self}, embedding risk-averse metrics as learning signals \cite{chow2018risk}, integrating Lyapunov functions into the model structure \cite{perkins2002lyapunov}, or a safety layer \cite{yi2023real} and shielding \cite{chen2022physics} that fixes or prevents unsafe actions.

Although safe learning can improve AI behaviour during training and action selection, an operator facing a deployment request still needs an independent way to verify whether the AI-enabled power grid operation system is safe. This verification step is often the final check before deployment, and is the core difference between verification (our main focus) and safe learning. Broadly speaking, verification aims to check the correctness of a system against a system specification expressed in mathematical logic, and AI verification is a subset of this broader area \cite{lindemann2025formal}. The strictest methods in this category aim to verify the system as a whole by exhaustively checking all possible system behaviours, and automated verification tools have been developed, such as model checking \cite{baier2000model} or automatic theorem proving \cite{shoukry2017smc}. Recent work has also developed a toolbox \cite{george2026torchlean} to formally verify neural networks in Lean \cite{avigad2015theorem}, a programming language and proof assistant based on type theory. These methods are theoretically sound and rigorous, but the computational complexity limits their applications on large-scale problems. In power systems, rigorous verification approaches verify the safe input region or the output range of a neural network, such as through mixed-integer optimisation formulations, which have been applied to verify neural-network classifiers for $N\!-\!1$ security, small-signal stability, and optimal power flow (OPF) surrogate~\cite{venzke2021verification,venzke2020learning}. Scalability continues to improve through bound propagation with approximations and GPU-accelerated tools, including alpha-beta-CROWN and power system-specific verification methods~\cite{wang2021beta,chevalier2024gpu,nellikkath2024scalable}. State-of-the-art GPU-accelerated methods can scale to neural networks with millions of parameters; in power system applications, the state-of-the-art alpha-beta-CROWN has been demonstrated to ML4ACOPF benchmarks with up to roughly 680k parameters and 402-dimensional inputs \cite{kaulen20256th}, and on grid-scale studies up to 300-bus direct current (DC) OPF and 500-bus transient-stability cases \cite{chevalier2024gpu, su2025neural}. However, considering millions of flexible grid-edge devices, each with its own AI model and complex power system coupling, scalability remains a challenge.


For power systems to benefit from AI-based coordination \emph{safely}, one computationally scalable approach for the near- to medium-term may be \emph{statistical verification}, which is a ``relaxed'' formal verification technique that allows checking a fraction of system behaviours \cite{lindemann2025formal}. It samples executions, evaluates whether each execution satisfies a prescribed property, and uses confidence intervals or hypothesis tests to make finite-sample statements about the satisfaction probability \cite{zarei2020statistical}. This process implies that the computational complexity can remain roughly the same order as AI model inference. Such a finite-sample guarantee provides evidence even when the available data are scarce, thereby making a stronger statement than a simple but commonly used evaluation on a held-out validation set, for which a guarantee can only be achieved asymptotically as the amount of data approaches infinity. Therefore, this paper focuses on statistical verification. Note that this concept is also related to scenario-based chance-constrained optimisation with probabilistic satisfaction guarantees~\cite{geng2019data}, but our objective here is more like chance-constraint \emph{evaluation} rather than optimisation. 

Statistical verification in power system applications is limited. In \cite{ling2026dynamic}, statistical verification is applied to power system optimisation methods but not AI approaches. Another work, \cite{cheng2023safeguarding}, applies a statistical model checking method to evaluate an RL-based power system controller, but the tightness of the probabilistic guarantee was not discussed, and the case study remains preliminary.
Conformal prediction is a popular statistical verification method~\cite{lindemann2025formal}, and has been applied to power system optimisation~\cite{stratigakos2026decision} and power system AI \cite{ellinas2026verification}, but it aims to generate a prediction region for the AI output given a (random) input, which is useful as an extra robust or trustworthiness layer of an AI model \cite{alcantara2026trustworthiness}. In contrast, our focus is on verifying the safety probability of the AI system itself, which could be a more direct certification. Note that the first stage of AI verification for power system applications will generally be performed on simulators, because it is unlikely that an unverified AI system can be directly tested on the real grid \cite[Sec.~2.3.1]{lemetayer2025evaluation}.

In addition to the lack of a systematic approach to statistical verification of power system AI, the other major gap is around the standard statistical verification itself. Theoretically, statistical guarantees are established with respect to the distribution of a held-out calibration dataset used to perform the verification, which can differ from the actual grid conditions due to various sources of distribution shifts, including changes in operating conditions, the inaccuracy of the grid simulator, or the impact of the implemented decision itself. Existing work has investigated distributionally robust certification that holds for a set of plausible distributions \cite{zhao2025distributionally}, but the distributional distance that controls the range of plausible distributions is often difficult to interpret physically.

This paper addresses the two major gaps in certifying AI-based grid-edge coordination with the following contributions:

\begin{itemize}
    \setlength{\itemsep}{0pt}
    \item We propose a systematic finite-sample probabilistic safety certification framework for AI-based grid-edge coordination. With a held-out calibration dataset of any size $N$, the method generates the \emph{tightest} upper confidence bound on the unsafe operation probability for any safety specification at a user-specified confidence level. Its dominant computational cost remains $N$ times of ordinary closed-loop inference on the calibration dataset.
    \item To address distribution shifts while remaining physically interpretable,
    we develop a worst-case probability certification approach motivated by adversarial attack concepts in AI \cite{goodfellow2014explaining}. Under a perturbation budget in the sample space (e.g., demand or generation scenarios) instead of the distribution space (harder to interpret physically), the proposed approach approximates the worst-case unsafe operation probability and provides more robust certification.
    \item We demonstrate the framework on an IEEE 123-bus grid-edge flexibility coordination case study with 1{,}000-agent AI decision models, each with different parameters, and include a realistic rolling-window safety certification study that verifies the effectiveness of the proposed framework.
\end{itemize}

The remainder of this paper is organised as follows. Section~\ref{sec:safety_certification_problem} defines the certification problem. Section~\ref{sec:finite_sample_safety_certificates} describes the proposed finite-sample certification framework. Section~\ref{sec:sample_space_robust_certification} integrates the adversarial attack into the certification to enhance robustness. Section~\ref{sec:case_studies} presents the case studies. Section~\ref{sec:discussion} discusses interpretation and limitations, and Section~\ref{sec:conclusion} concludes the paper.

\section{Safety Certification Problem}\label{sec:safety_certification_problem}

At a high level, the AI-based grid-edge coordination can be abstracted as an input--AI--decision process. When a grid operator assesses system safety, a natural question is: given a set of input scenarios that the grid may encounter, what control decisions will the AI system generate, and are these decisions safe for the grid? The operator may already have scenarios that are considered representative of actual operating conditions, as scenario-based testing is common in grid operation, e.g., \textit{PJM performs “what if” scenarios all day long to assess conditions based on data from hundreds of thousands of points on the grid every four seconds \cite{pjm_system_operations}.} However, the key challenge is how such a dataset can be used to derive a rigorous statistical guarantee, which is not fully addressed by current grid operations. 

Note that this paper models a general AI-based grid-edge coordination problem. In the abstracted input--AI--decision process,
the input scenario can be general. It may be a single operating point, a system initial state, or a tuple containing the system initial state and multivariate time series, such as locational demand, renewable generation, temperature, device availability, and contingency events. The AI system can also be arbitrary, e.g., different neural network structures or decision rules trained by RL, imitation learning, self-supervised learning, or other AI methods. The  AI system can process the input scenarios in a one-shot generation style, a sequential recursive manner, or another workflow.

\subsection{Mathematical Formulation}
Let $W$ denote a random operating input scenario sampled from the unknown operating distribution $P_0$ of interest. Let $\pi$ denote the AI system, which maps the observed system condition in $W$ to control decisions. We assume that the system operator can determine whether the AI control decisions are safe for a pre-set safety threshold $\nu$, through
\begin{equation}
    L_\nu (\pi, W) \coloneqq  \one(M(\pi, W) \geq \nu),
\end{equation}
with $L_\nu (\pi, W)\in \{0,1\}$ the binary safe/unsafe outcome (1 represents unsafe and 0 otherwise), and $M(\pi, W) \in \mathbb{R}$ a real-valued safety score with an operator-defined threshold $\nu$. This is a practical assumption, as the power industry already uses a suite of analysis tools for power system safety across different time scales, from electromagnetic transients to quasi-steady-state power flow and from normal conditions to contingencies and faults. The safety score $M(\pi, W)$ can be the line current, nodal voltage, frequency nadir, short-circuit current magnitude, etc., and can be the maximum of multiple metrics if desired. The calculation of $M(\pi, W)$ and thus $L_\nu(\pi, W)$ should be performed on simulators, as testing an unverified AI system on the real grid is unlikely to be acceptable. Simulators inevitably introduce an uncertain sim-to-real gap, which can be viewed as epistemic uncertainty and investigated through the proposed adversarial attack in Section~\ref{sec:sample_space_robust_certification}.

With this in mind, let $\calH$ denote the simulator that evaluates grid safety so that a complete input--AI--decision--evaluation workflow can be expressed as:
\begin{equation}
    W \xrightarrow{\;\pi\;} \pi(W)
    \xrightarrow{\;\calH\;} M(\pi, W) \xrightarrow{\;\nu\;} L_\nu(\pi,W)\in\{0,1\}.
    \label{eq:threshold_safety_process}
\end{equation}
Consider an illustrative example in which a multi-agent AI system is deployed to manage grid-edge devices, reduce device owners' energy bills, and respect distribution network constraints. As this is an energy management application, network safety can be checked primarily by running quasi-steady-state power flow given the AI control decisions to obtain every nodal phase voltage $V_{i,\phi,t}(\pi,W)$ for node $i$, phase $\phi$, and time step $t$. Because distribution networks are often limited by voltage constraints, one can define a threshold-based safety outcome indicating whether the maximum voltage magnitude deviation over the certification horizon exceeds a prescribed threshold:
\begin{equation}
    L_\nu(\pi,W)=
    \one\Big\{ \max_{t,i,\phi}
    \left||V_{i,\phi,t}(\pi,W)|-1\right|\geq\nu\Big\}.
    \label{eq:voltage_safety_outcome}
\end{equation}
In this voltage example, the threshold can be $\nu=0.03$ p.u., $0.05$ p.u. or $0.1$ p.u., depending on the safety specification of the distribution network operator. 

\subsection{Calibration Setup and Objective}

A core step in statistical verification is to obtain a dataset for the property to be verified. Here, this refers to the binary safety outcome $L_\nu(\pi,W)$. The dataset used to evaluate $L_\nu(\pi,W)$ should ideally have the same distribution as future grid operating conditions. Based on Eq.~\eqref{eq:threshold_safety_process}, the dataset for $L_\nu(\pi,W)$ can be constructed from input scenario data. Let $W^1,\ldots,W^n$ denote independent random samples drawn from the distribution of the input $W$, namely $P_0$. We collect these samples into the calibration dataset
\begin{equation}
    D_n\coloneqq (W^1,\ldots,W^n)\sim P_0^n .
    \label{eq:calibration_dataset}
\end{equation}
The calibration dataset is itself a random object. Thus, we need Assumption~\ref{ass:stated_calibration_distribution} to regulate the sampling behaviour.
\begin{assumption}\label{ass:stated_calibration_distribution}
The calibration scenarios $W^1,\ldots,W^n$ are independent and identically distributed (i.i.d.) samples from $P_0$, which is the distribution of future operating scenarios against which the AI system is to be verified.
\end{assumption}
Note that Assumption~\ref{ass:stated_calibration_distribution} imposes independence only across calibration scenarios. The components within each scenario $W^i$ may exhibit spatial and temporal dependence.

The statistical verification target is the unsafe operation probability:
\begin{equation}
    p_\pi=\Prob_{W\sim P_0}\left(L_\nu(\pi,W)=1\right).
    \label{eq:true_binary_probability}
\end{equation}
The corresponding safe operation probability is $1-p_\pi$. This probability answers the question posed at the beginning of the section: \emph{given a set of input scenarios that the grid may encounter, specified by $P_0$, what is the probability that the AI system is unsafe for the grid?}

\section{Finite-Sample Probabilistic Safety Certificates}
\label{sec:finite_sample_safety_certificates}

\subsection{Bernoulli Reduction and Binomial Conversion}

By the definition of $p_\pi$ in \eqref{eq:true_binary_probability}, the certification target is the probability that the binary safety outcome equals one. This is the key step that reduces the problem's complexity: although the operating-scenario distribution $P_0$ over $W$ may be complex and difficult to characterise directly, the safety-certification problem only depends on the induced distribution of $L_\nu(\pi,W)$, which is Bernoulli by definition. Evaluating $L_\nu(\pi,W^j)$ over the calibration dataset then provides repeated draws from this Bernoulli distribution. The unsafe-outcome count is
\begin{equation}
    K_n(D_n)=\sum_{j=1}^n L_\nu(\pi,W^j).
    \label{eq:violation_count}
\end{equation}
Under Assumption~\ref{ass:stated_calibration_distribution}, we have that
\begin{equation}
    K_n(D_n)\sim \Bin(n,p_\pi).
    \label{eq:binomial_model}
\end{equation}
Therefore, the certification becomes a finite-sample inference for a random variable with a binomial distribution. 

\subsection{One-Sided Exact Upper Certificate}

For a realised calibration dataset, the count $K_n(D_n)=k$ is a single observation from the binomial distribution in~\eqref{eq:binomial_model}. Clopper--Pearson (CP) confidence intervals provide exact finite-sample confidence bounds for the Bernoulli probability. From the safety certification perspective, the relevant quantity is an upper bound on the unsafe operation probability, so we use the one-sided CP upper confidence bound~\cite{clopper1934use,cai2005one}
\begin{equation}
    U(k;n,\alpha)=
    \begin{cases}
    1, & k=n,\\
    \BetaInv(1-\alpha;k+1,n-k), & k<n,
    \end{cases}
    \label{eq:cp_upper}
\end{equation}
where $\BetaInv(q;a,b)$ denotes the $q$-quantile of the Beta distribution with shape parameters $(a,b)$.

\begin{theorem}\label{thm:pointwise_finite_sample_safety_certificate}
For any fixed AI system $\pi$, calibration size $n$, and confidence parameter $\alpha\in(0,1)$,
\begin{equation}
    \Prob_{D_n\sim P_0^n}\left(p_\pi\leq U(K_n(D_n);n,\alpha)\right)\geq 1-\alpha .
    \label{eq:cp_coverage}
\end{equation}
\end{theorem}
\begin{proof}
For fixed $\pi$, $K_n(D_n)$ follows the binomial model in \eqref{eq:binomial_model}. Eq.~\eqref{eq:cp_coverage} is the coverage property of the one-sided Clopper--Pearson interval~\cite{clopper1934use,cai2005one}.
\end{proof}
The interpretation of the probability in \eqref{eq:cp_coverage} is as follows: if this certification procedure is repeated $K$ times, with a fresh calibration dataset $D_n$ drawn in each experiment, then the inequality $p_\pi\leq U(K_n(D_n);n,\alpha)$ holds in at least a fraction $1-\alpha$ of the repetitions in the long run. This is roughly equivalent to saying that at least a $1-\alpha$ proportion (in the long run) of the certifications made by a system operator will be correct.

Moreover, our proposed certificate through the one-sided CP upper confidence bound is the \emph{tightest} possible:
\begin{theorem}
The one-sided CP upper confidence bound is the smallest among all non-decreasing upper confidence bounds (i.e., more observed unsafe events lead to a higher or unchanged bound) with finite-sample coverage of at least \(1-\alpha\) \cite[Theorem~2]{wang2006smallest}.
\end{theorem}
It is worth noting that the two-sided CP bound can be very conservative, but our safety certificate only requires the one-sided CP upper bound on the unsafe operation probability, which, in contrast, turns out to be the tightest (see Remark 1 of Ref.~\cite{wang2006smallest}).

Operators often need an accept/reject decision relative to a prescribed risk tolerance $\varepsilon$. The desired assertion is
\begin{equation}
    p_\pi\leq \varepsilon.
    \label{eq:operator_assertion}
\end{equation}
This assertion can be made with the following statistical guarantee. 
\begin{assertionprocess}\label{proc:assertion}
Given a realized calibration dataset $D_n$, if $U(K_n(D_n);n,\alpha)\leq \varepsilon$, one can assert that $p_{\pi}\leq \varepsilon$ at confidence level $1-\alpha$.
\end{assertionprocess}
The confidence level has the following formal interpretation.
\begin{corollary}\label{cor:false_cert}
For any fixed $\pi$, if $p_\pi>\varepsilon$, then
\begin{equation}
    \Prob_{D_n\sim P_0^n}\left(U(K_n(D_n);n,\alpha)\leq\varepsilon\right)\leq\alpha .
    \label{eq:false_certification}
\end{equation}
\end{corollary}
\begin{proof}
Since $p_\pi>\varepsilon$, we have that $$\left\{U(K_n(D_n);n,\alpha)\leq\varepsilon\right\}
\subseteq
\left\{U(K_n(D_n);n,\alpha)<p_\pi\right\}.$$ 

\noindent By Theorem~\ref{thm:pointwise_finite_sample_safety_certificate}, $\Prob_{D_n\sim P_0^n}
\left(U(K_n(D_n);n,\alpha)\leq\varepsilon\right)
\leq \alpha.$
\end{proof}
Note that neither $p_\pi$ nor $\varepsilon$ is random. The confidence level refers to the data-driven procedure that decides whether to issue the assertion.
The risk tolerance should also be specified before inspecting the calibration outcomes or safety certification map to avoid post-hoc selection ($p$-hacking).

Finally, the proposed verification upper bound can converge to the true probability:
\begin{theorem}
\label{thm:cp_asymptotic_consistency}
For a fixed AI system $\pi$ and a confidence parameter $\alpha\in(0,1)$, under
Assumption~\ref{ass:stated_calibration_distribution}, we have
\begin{equation}
    U(K_n(D_n);n,\alpha)
    \xrightarrow[n\to\infty]{\mathrm{a.s.}} p_\pi .
    \label{eq:cp_upper_asymptotic_consistency}
\end{equation}
\end{theorem}

\begin{proof}
Let $\widehat{p}_n \coloneqq  \frac{K_n(D_n)}{n}$ be the empirical unsafe operation frequency. We first show that the CP upper bound is asymptotically close to $\widehat{p}_n$.

For a realized count $k<n$, let
\begin{equation*}
    B_{n,k}\sim \mathrm{Beta}(k+1,n-k).
\end{equation*}
By the definition of the one-sided CP upper bound in~\eqref{eq:cp_upper},
\begin{equation*}
    U(k;n,\alpha)=Q_{1-\alpha}(B_{n,k}),
\end{equation*}
where $Q_{1-\alpha}(B_{n,k})$ denotes the $(1-\alpha)$-quantile of
$B_{n,k}$. The mean and variance of $B_{n,k}$ are
\begin{align*}
    m_{n,k}
    &\coloneqq  \mathbb{E}[B_{n,k}]
      = \frac{k+1}{n+1},\\
    \sigma_{n,k}^2
    &\coloneqq  \mathrm{Var}(B_{n,k})
      = \frac{(k+1)(n-k)}{(n+1)^2(n+2)}
      \leq \frac{1}{4(n+2)} .
\end{align*}
By Chebyshev's inequality,
\begin{equation*}
    Q_{1-\alpha}(B_{n,k})
    \leq
    m_{n,k}
    +
    \frac{\sigma_{n,k}}{\sqrt{\alpha}},
\end{equation*}
and similarly,
\begin{equation*}
    Q_{1-\alpha}(B_{n,k})
    \geq
    m_{n,k}
    -
    \frac{\sigma_{n,k}}{\sqrt{1-\alpha}} .
\end{equation*}
Therefore,
\begin{equation}
    \left|
        U(k;n,\alpha)-m_{n,k}
    \right|
    \leq
    \frac{c_\alpha}{2\sqrt{n+2}},
    \label{eq:cp_beta_quantile_bound}
\end{equation}
where $c_\alpha \coloneqq \max\left\{\alpha^{-1/2},(1-\alpha)^{-1/2}\right\}$. Moreover,
\begin{equation}
    \left|
        m_{n,k}-\frac{k}{n}
    \right|
    =
    \frac{n-k}{n(n+1)}
    \leq
    \frac{1}{n+1}.
    \label{eq:beta_mean_empirical_gap}
\end{equation}
Combining \eqref{eq:cp_beta_quantile_bound} and
\eqref{eq:beta_mean_empirical_gap}, we obtain, for all $k<n$,
\begin{equation}
    \left|
        U(k;n,\alpha)-\frac{k}{n}
    \right|
    \leq
    \frac{1}{n+1}
    +
    \frac{c_\alpha}{2\sqrt{n+2}} .
    \label{eq:uniform_cp_empirical_gap}
\end{equation}
If $k=n$, then $U(n;n,\alpha)=1=k/n$, so
\eqref{eq:uniform_cp_empirical_gap} also holds. Hence, we can conclude that
\begin{equation}
    \left|
        U(K_n(D_n);n,\alpha)-\widehat{p}_n
    \right|
    \leq
    \dfrac{1}{n+1}
    +
    \dfrac{c_\alpha}{2\sqrt{n+2}}
    \xrightarrow[n\to\infty]{}0.
    \label{eq:cp_close_to_empirical}
\end{equation}

Under Assumption~\ref{ass:stated_calibration_distribution}, the variables $L_\nu(\pi,W^j)$ are i.i.d. Bernoulli random variables with mean $p_\pi$. By the strong law of large numbers,
\begin{equation}
    \widehat{p}_n
    =
    \frac{1}{n}\sum_{j=1}^n L_\nu(\pi,W^j)
    \xrightarrow[n\to\infty]{\mathrm{a.s.}} p_\pi .
    \label{eq:slln_empirical_probability}
\end{equation}
The triangle inequality, together with
\eqref{eq:cp_close_to_empirical} and \eqref{eq:slln_empirical_probability},
gives
\begin{equation*}
    U(K_n(D_n);n,\alpha)
    \xrightarrow[n\to\infty]{\mathrm{a.s.}} p_\pi .
\end{equation*}
\end{proof}

\section{Sample-Space Robust Certification}
\label{sec:sample_space_robust_certification}

The nominal certificate characterises the unsafe operation probability with a finite-sample guarantee, but the operator may find such a guarantee insufficient for several reasons:
\begin{enumerate}
    \item temporal distribution shifts, such as evolving weather patterns, may make future operation deviate from the calibration distribution;
    \item the simulator-based evaluation tool ($\calH$ in Eq.~\eqref{eq:threshold_safety_process}) introduces bias, because the physical parameters used to model the power grid may deviate from the real system (and all models are wrong);
    \item if the calibration data are collected in the real world, inherent temporal dependence may break the i.i.d. assumption;
    \item due to the lack of interpretability of AI, the operator may want stronger stress tests on the model.
\end{enumerate}
To handle these issues, we enhance the safety guarantee with an additional robustness layer. Although one may argue that system operators may already have prepared a set of ``worst-case'' scenarios for testing, the actual worst case for an AI system can be highly model-specific because of the non-linearity of neural networks, especially when the system is large.
Existing work has investigated distributionally robust certification, where the certificate is required to hold for a set of plausible distributions \cite{zhao2025distributionally}. However, their physical interpretability is limited, because the distributional distance that controls the range of plausible distributions does not have a direct connection with real-world scenarios.
These methods may also involve complex sampling approaches. To tackle this issue, we borrow the idea of adversarial attacks in AI. More specifically, given a perturbation budget $\rho$, for each realisation $W^j$ in the calibration set, we want to find a neighbourhood scenario that leads to the most unsafe outcome, namely
\begin{equation}
    \widetilde W_\rho^j
    \in
    \operatorname*{arg\,max}_{\omega\in\calW}
    M(\pi,\omega)
    \quad
    \mathrm{s.t.}\quad
    \|\omega-W^j\|\leq\rho .
    \label{eq:adversarial_neighborhood}
\end{equation}
Here $\|\cdot\|$ is a general norm or distance function. It can be selected so that the perturbation has a clear physical interpretation: for example, solar generation may deviate by at most 10\%, or the network may experience at most one or $k$ contingencies. This gives the perturbation budget an engineering meaning that system operators can understand.
The optimisation problem \eqref{eq:adversarial_neighborhood} has the form of an adversarial attack. It is hard to solve because of the complex AI system and grid physics involved. However, a large body of work on adversarial-attack methods provides several approximation options that work well empirically \cite{ren2020adversarial}.

\begin{enumerate}
    \item Gradient-based optimisation can be used if the whole chain in Eq.~\eqref{eq:threshold_safety_process} is differentiable. One effective and computationally tractable method is projected gradient descent (PGD), which finds the neighbourhood by iteratively performing
\begin{equation}
    \omega_{t+1}^j
    =
    \Pi_{\rho,j}
    \left[
    \omega_t^j
    +
    \eta\,\operatorname{sign}
    \left(\nabla_{\omega}L_\nu(\pi,\omega_t^j)\right)
    \right],
    \label{eq:pgd_update}
\end{equation}
where $\omega_0^j=W^j$, $\eta$ is the step size, $\Pi_{\rho,j}$ denotes projection onto $\{\omega\in\calW:\|\omega-W^j\|\leq\rho\}$, and the final iterate is used as the adversarial candidate $\widetilde W_\rho^j$, i.e., the approximate solution to \eqref{eq:adversarial_neighborhood}. Note that for AC power flow, the gradient can still be calculated through implicit differentiation \cite{zhou2026gradmap}.

    \item Heuristic methods or RL, such as policy-gradient search, can be used if the whole chain in Eq.~\eqref{eq:threshold_safety_process} is non-differentiable. Random perturbation is also a simple special case.
\end{enumerate}

In addition, regardless of the method used, one must check whether the calculated neighbourhood scenario is less safe than the original realisation $W^j$. Otherwise, the original scenario should be kept. Thus, for an exact or approximate adversarial candidate $\widetilde W_\rho^j$, we define the selected scenario as
\begin{equation}
    \widehat W_\rho^j
    =
    \begin{cases}
    \widetilde W_\rho^j, & L_\nu(\pi,\widetilde W_\rho^j)\geq L_\nu(\pi,W^j),\\
    W^j, & \text{otherwise}.
    \end{cases}
    \label{eq:adversarial_retained_sample}
\end{equation}

After this process, we create an adversarial calibration set
\begin{equation}
    \widehat D_{n,\rho}\coloneqq (\widehat W_\rho^1,\ldots,\widehat W_\rho^n),
    \label{eq:adversarial_calibration_set}
\end{equation}
after which the same process in Section~\ref{sec:finite_sample_safety_certificates} can be applied to obtain the corresponding upper bound on the unsafe operation probability. Theorem~\ref{thm:adversarial_cp_dominance} shows that this upper bound is no smaller than the nominal CP upper bound.
\begin{theorem}
\label{thm:adversarial_cp_dominance}
For any realized calibration dataset $D_n$, perturbation budget $\rho\geq0$, and selected adversarial set $\widehat D_{n,\rho}$ defined by \eqref{eq:adversarial_retained_sample}--\eqref{eq:adversarial_calibration_set}, we have that
\begin{equation}
    U(\widehat{K}_{n,\rho}(\widehat D_{n,\rho});n,\alpha)
    \geq
    U(K_n(D_n);n,\alpha).
    \label{eq:adversarial_cp_dominance}
\end{equation}
\end{theorem}
\begin{proof}
By \eqref{eq:adversarial_retained_sample}, each selected adversarial scenario satisfies
\begin{equation}
    L_\nu(\pi,\widehat W_\rho^j)\geq L_\nu(\pi,W^j).
    \label{eq:retained_sample_dominance}
\end{equation}
The summation of~\eqref{eq:retained_sample_dominance} over $j=1,\ldots,n$ gives $\widehat{K}_{n,\rho}(\widehat D_{n,\rho})\geq K_n(D_n)$. Since the one-sided CP upper bound $U(k;n,\alpha)$ is non-decreasing in the observed unsafe-outcome count $k$, we can conclude the inequality~\eqref{eq:adversarial_cp_dominance}.
\end{proof}

\section{Case Studies}\label{sec:case_studies}

We present case studies based on the IEEE 123-bus medium-voltage feeder \cite{ieee_feeder} with 1{,}000 agents. Fig.~\ref{fig:feeders} presents the network used in the case studies. Each agent manages a single energy device, such as a battery, heat pump, or generator. The device models follow Ref.~\cite{zhou2026gradmap}. The training objective of the multi-agent AI system is to minimise each home's energy bill under a time-varying tariff while minimising distribution network voltage violations at each time step (with a threshold set to $0.05$ p.u.) through the primal-dual learning framework \cite{zhou2026gradmap}. 
The control time step is 15 minutes. At this time scale, safety is primarily determined by the distribution network voltage magnitude calculated by solving the three-phase unbalanced AC load flow, as defined in Eq.~\eqref{eq:voltage_safety_outcome}.

\begin{figure}[tb]
    \centering
    \vspace{-6mm}
    \includegraphics[width=1\columnwidth]{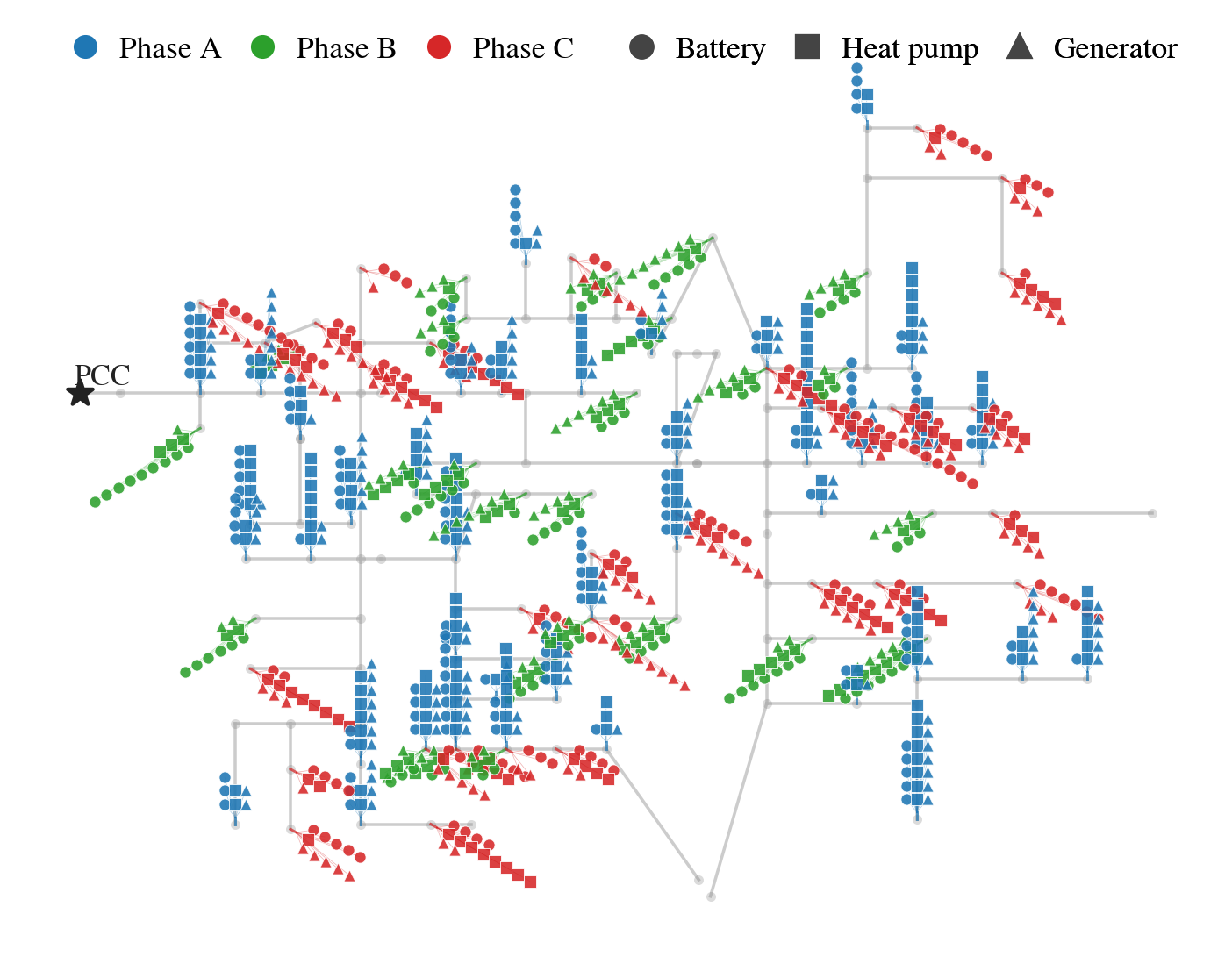}
    \vspace{-12mm}
    \caption{IEEE 123-bus feeder topology used in the two case studies, together with the connection locations of the $1{,}000$ agents. Gray nodes denote feeder buses, coloured branches indicate occupied phase connections, and marker types distinguish the $334$ battery, $333$ heat-pump, and $333$ generator agents distributed across the network.}
    \label{fig:feeders}
\end{figure}

All decentralised policies use the same 8-dimensional local observation at the \emph{current} time step, defined as
\begin{equation}
    \vect{o}_{i,t}
    =
    [
    \bar{t},\,
    \bar{x}_{i,t},\,
    \bar{d}_{i,t}^{\mathrm{net}},\,
    \bar{p}_t^{\mathrm{imp}},\,
    \bar{p}_t^{\mathrm{exp}},\,
    \bar{\theta}_t^{\mathrm{out}},\,
    \bar{u}_{i,t},\,
    \bar{v}_{i,t}^{\mathrm{loc}}
    ]^{\top},
    \label{eq:local_obs}
\end{equation}
where the entries are the normalised time $\bar{t}\coloneqq t/T$, device state $\bar{x}_{i,t}$ (energy level for batteries, temperature for heat pumps, and power at the previous time step for generators), net uncontrollable demand $\bar{d}_{i,t}^{\mathrm{net}}$ (uncontrollable load minus solar generation), import and export prices $\bar{p}_t^{\mathrm{imp}}$ and $\bar{p}_t^{\mathrm{exp}}$, outdoor temperature $\bar{\theta}_t^{\mathrm{out}}$, terminal-state urgency $\bar{u}_{i,t}$, and local voltage $\bar{v}_{i,t}^{\mathrm{loc}}$. All features are normalised to the interval $[0,1]$ using either the device parameters or the training data statistics. The terminal-state urgency feature $\bar{u}_{i,t}$ was found to be helpful for satisfying terminal-state constraints of individual devices, such as terminal SoC requirements. It is calculated as the remaining difference required to satisfy a terminal-state constraint, such as the end-of-horizon SoC, divided by the remaining number of time steps.

For training and testing, three-phase unbalanced AC power flow is solved using the Z-Bus method \cite{zbus}. Different AI and non-AI algorithms are evaluated for their safety levels, including the recent scalable gradient-based AI learning method GradMAP, which outperforms standard multi-agent RL algorithms \cite{zhou2026gradmap}, the multi-agent extension of the self-supervised learning approach GradMA \cite{zhou2026gradmap, park2024self}, the standard multi-agent RL algorithm independent proximal policy optimisation (IPPO), and a simple non-AI rule-based algorithm (SuperNaive). Under the SuperNaive algorithm, the
batteries remain idle, the heat pumps always maintain the
indoor temperature at the reference level, and the generators always operate at 100\% power. 
We test IPPO instead of the better-performing MAPPO because MAPPO's memory consumption exceeds our hardware limit. 
Each agent is a neural network with a single hidden layer of 16 neurons (two layers in total) and tanh activation, which performs better empirically than deeper or wider fully connected networks \cite{zhou2026gradmap}. Although each network is small, there are 1{,}000 independent agents in total with no parameter sharing, and we consider the exact nonlinear, non-convex, three-phase unbalanced network constraints that couple these agents. Detailed hyperparameters for AI training follow exactly \cite{zhou2026gradmap}.

To better evaluate the proposed safety verification method, we carried out two case studies. The first is a case in which Assumption~\ref{ass:stated_calibration_distribution} is satisfied by construction. This verifies the stated behaviours of the proposed method when the required conditions are met. To achieve this, we use the HEDGE tool \cite{charbonnier2024hedge}, a generative machine learning model trained on real-world residential data, to generate residential demand and solar data for training the AI agents, AI verification, and out-of-sample safety testing. Fig.~\ref{fig:hedge_profile_overview} visualises the generated demand and solar profiles. We also create a synthetic temperature profile for heat pumps. We use HEDGE to generate mutually independent daily demand and solar trajectories. The datasets for training, calibration, and testing contain $365$, $10^4$, and $10^4$ daily trajectories, respectively. Fig.~\ref{fig:dispatch_visualization} shows that the multi-agent AI system learns reasonable behaviour that responds to the price signal while maintaining the voltage limit for one day. For this case study, the safety outcome is the per-day maximum voltage deviation over all $96$ $15$-minute control steps. Thus, $W$ defined in Section~\ref{sec:safety_certification_problem} refers to one complete daily trajectory, including the daily demand, renewables, temperature and price signals, together with the relevant closed-loop network and device-state information. 

The second case study represents a more realistic setup for testing the empirical performance of the proposed method and highlights the importance of the proposed adversarial attack method for improving its practical usefulness. This is a rolling-window test in which, for each window, we use the most recent six months of data. The first five months are used for training, the sixth month is used for AI verification, and the verification result is evaluated against out-of-sample testing outcomes in the next month. The verification and testing sets use the raw daily trajectories from the corresponding months, so the sample sizes are the numbers of days in those months. No replacement resampling is used. In this setting, the temporal correlation of the time-series data exists, andAssumption~\ref{ass:stated_calibration_distribution} does not hold. However, this is a setting in which a practitioner might directly apply our statistical verification technique. This case study uses only real-world time-series data. The household load data come from~\cite{wardle2020dataset}, the solar-generation and air-temperature traces come from~\cite{renewable_ninja}, and the import/export prices come from the UK Octopus Energy Agile tariff, which varies at every time step and on every day~\cite{octopus_agile_historical_data}, as illustrated in Fig.~\ref{fig:real_profile_overview}. 

The implementation is written in JAX 0.4.38 and Python 3.10.19, and all experiments were run on a single NVIDIA RTX PRO 5000 Blackwell 48-GB GPU.

\subsection{Case Study I: Nominal Certification}

Fig.~\ref{fig:nominal_certification} shows the certification results using $n=100$ randomly selected daily calibration scenarios under different safety thresholds $\nu$. Another widely used approach to validating model performance is to calculate the empirical unsafe probability on a held-out validation set (the calibration set here), which we also calculate for comparison and denote as the calibration empirical probability. As can be seen, the calibration empirical probability is lower than the out-of-sample unsafe probability for most cases across GradMA, IPPO, and SuperNaive, highlighting that using only the validation-set empirical metric as the performance estimate is less reliable. In contrast, the nominal CP certificate provides a statistically valid upper confidence bound on the unsafe probability under the stated calibration distribution, which highlights the importance of rigorous statistical safety verification instead of simplified evaluation on a held-out validation set. However, as the voltage threshold increases, the verification probability cannot decrease below around 0.03 due to the limited amount of calibration data.

\subsection{Case Study I: Sample-Size Effect}

This section investigates whether increasing the amount of calibration data can lead to a tighter verification probability, as stated in Theorem~\ref{thm:cp_asymptotic_consistency}. In Fig.~\ref{fig:sample_complexity}, as $n$ increases, despite some variation, the proposed verification probabilities tighten toward the out-of-sample empirical probabilities. Furthermore, this figure highlights the effectiveness of the proposed method with small datasets and shows that the standard empirical validation metric can be unreliable even with more than $1{,}000$ samples. 

\begin{figure}[t]
\centering
\includegraphics[width=\linewidth]{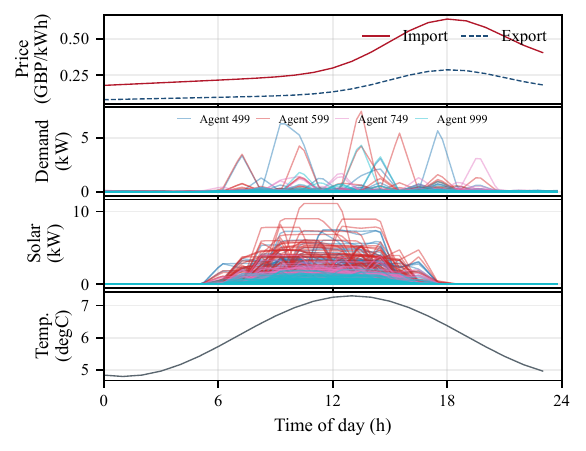}
\vspace{-8mm}
\caption{Synthetic time-series data generated by HEDGE for Case Study I.}
\label{fig:hedge_profile_overview}
\end{figure}

\begin{figure}[t]
\centering
\includegraphics[width=\linewidth]{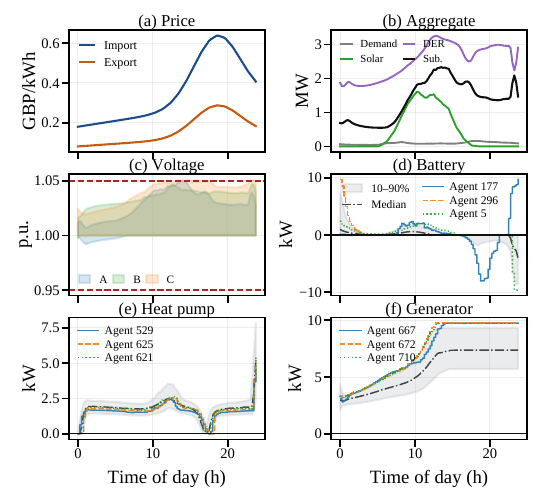}
\vspace{-8mm}
\caption{Dispatch results of the GradMAP multi-agent AI system on a typical held-out day for Case Study I. Dashed red voltage lines mark the 0.95 and 1.05 p.u. limits.}
\label{fig:dispatch_visualization}
\end{figure}

\begin{figure}[t]
\centering
\includegraphics[width=\linewidth]{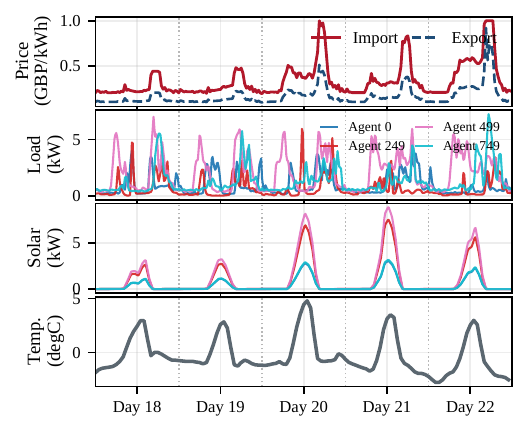}
\vspace{-8mm}
\caption{Real-world time-series data for Case Study II.}
\label{fig:real_profile_overview}
\end{figure}

\begin{figure}[t]
\centering
\includegraphics[width=\linewidth]{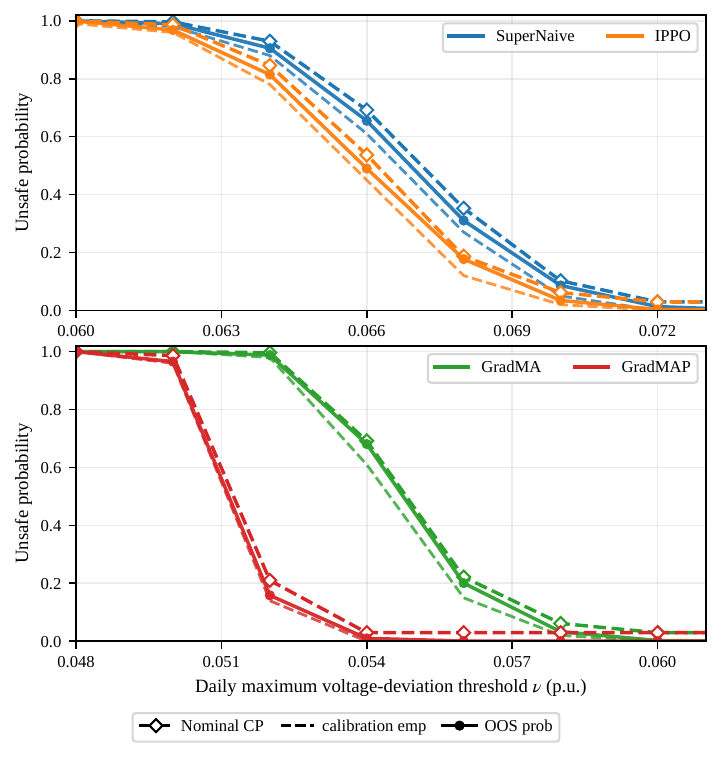}
\caption{Per-day maximum voltage safety certification curves for Case Study I with $n=100$ calibration scenarios and $\alpha=0.05$. The upper panel compares SuperNaive and IPPO, while the lower panel compares GradMA and GradMAP. Nominal CP refers to the proposed method in Section \ref{sec:finite_sample_safety_certificates}, calibration emp refers to the frequency of unsafe scenarios evaluated on the calibration set (the commonly used held-out method), and OOS prob refers to the true unsafe probability, which is estimated by evaluating on 10{,}000 out-of-sample scenarios. Note that the AI methods are trained to minimise the maximum voltage deviation above $0.05$ p.u.}
\label{fig:nominal_certification}
\end{figure}

\begin{figure}[t]
\centering
\includegraphics[width=\linewidth]{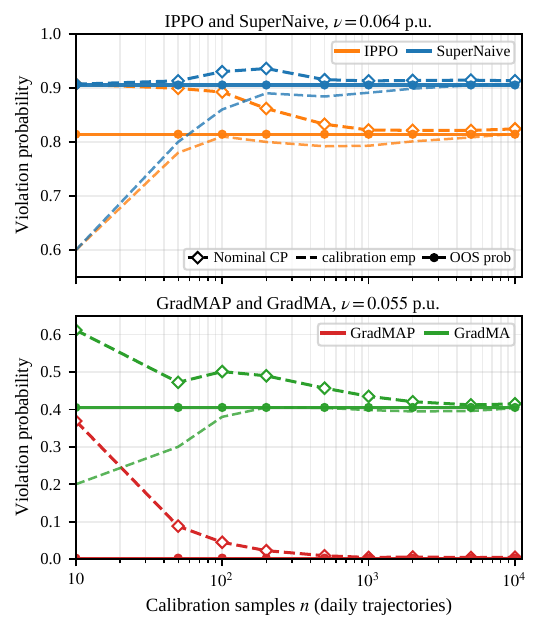}
\vspace{-8mm}
\caption{Effect of calibration sample size with $\alpha=0.01$. Nominal CP refers to the proposed method in Section \ref{sec:finite_sample_safety_certificates}, calibration emp refers to the frequency of unsafe scenarios evaluated on the calibration set (the commonly used held-out method), and OOS prob refers to the true unsafe probability, which is estimated by evaluating on 10,000 out-of-sample scenarios.}
\label{fig:sample_complexity}
\end{figure}

\begin{figure}[t]
\centering
\includegraphics[width=\linewidth]{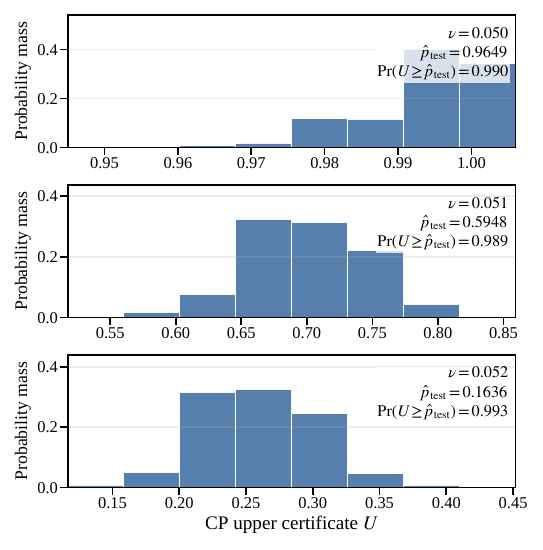}
\vspace{-8mm}
\caption{Monte Carlo distributions of the nominal CP upper certificate $U$ over 1{,}000 calibration repetitions for selected voltage thresholds $\nu$ under $\alpha=0.01$. $\hat{p}_\text{test}$ is the estimated true unsafe probability by evaluating on 10,000 out-of-sample scenarios. $\mathrm{Pr(U\geq \hat{p}_\text{test})}$ represents the proportion of the 1{,}000 calibration scenarios where $U\geq \hat{p}_\text{test}$.}
\label{fig:nominal_cp_audit_u_histograms}
\end{figure}

\subsection{Case Study I: Confidence Level and Statistical Assertion}

We also evaluate whether the stated confidence level $1-\alpha$ is achieved. To do so, we repeat the certification 1{,}000 times to obtain 1{,}000 data-dependent $U(K_n(D_n);n,\alpha)$ for the selected GradMAP AI system, each time using a different HEDGE-generated calibration set. For each of the 1{,}000 calibration sets, we estimate the out-of-sample unsafe probability using a separate independent 10{,}000-day HEDGE-generated test set. Fig.~\ref{fig:nominal_cp_audit_u_histograms} shows the distributions of the 1{,}000 $U$ values for representative voltage thresholds. As can be seen, for each threshold, the proportion of $U$ values that are not less than the out-of-sample unsafe probability matches or exceeds the specified confidence level $1-\alpha$. Also note that as specified by Theorem \ref{thm:cp_asymptotic_consistency} and illustrated in Fig.~\ref{fig:sample_complexity}, the variance of the distribution of $U$ should diminish as the size of calibration data increases.

\subsection{Case Study I: Robust Certificate}

We also test the robustness of certificates under different perturbation budgets $\rho$ defined in Eq.~\eqref{eq:adversarial_neighborhood}. We use the $\ell_\infty$ norm, so for budget $\rho$ the adversarial-attack search space is:
\begin{align*}
    \widetilde{d}_{a,t} &\in [(1-\rho)d_{a,t},(1+\rho)d_{a,t}],\\
    \widetilde{r}_{a,t} &\in [(1-\rho)r_{a,t},(1+\rho)r_{a,t}].
\end{align*}
The PGD attack was used to approximate the solution to \eqref{eq:adversarial_neighborhood}.

Fig.~\ref{fig:pgd_stress_map}(a) shows the certificate results under the nominal and robust cases. It is expected that the certified unsafe probability increases as the perturbation budget increases. By Assertion process~\ref{proc:assertion}, the heatmap in Fig.~\ref{fig:pgd_stress_map}(a) can be converted into a more operator-facing binary pass/fail statistical assertion at different unsafe thresholds, as shown in Fig.~\ref{fig:pgd_stress_map}(b)--(d). Such visualisation can provide operators with a comprehensive understanding of AI-system safety under different safety thresholds, confidence levels, and robustness budgets. 
Also note that, as shown in Fig.~\ref{fig:pgd_stress_map}(a), increasing the confidence level from 95\% to 99.9\% only slightly increases the certificate values $U$, but the small value change exceeds the acceptance threshold $\epsilon=0.001$, causing all cases in the second row of Fig.~\ref{fig:pgd_stress_map}(b) to fail.

\begin{figure*}[t]
\centering
\includegraphics[width=\textwidth]{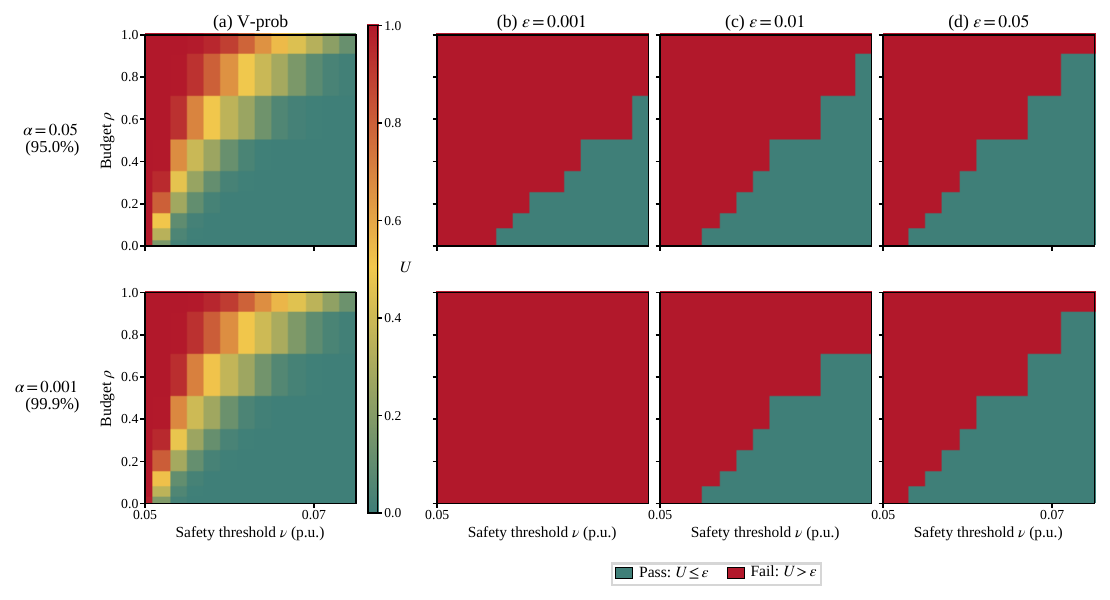}
\vspace{-8mm}
\caption{Safety certification map for Case Study I. Plot (a) shows the CP upper certificate $U$ in the nominal or adversarial robust cases. Plots (b)--(d) apply Assertion process~\ref{proc:assertion} with $\varepsilon=0.001$, 0.01, and 0.05; green cells pass $U\leq\varepsilon$ and red cells fail. The horizontal axis is the per-day maximum voltage-deviation threshold $\nu$ from 0.05 to 0.075 p.u.; the vertical axis is the fractional load/PV perturbation budget $\rho$.}
\label{fig:pgd_stress_map}
\end{figure*}

\subsection{Case Study II}

Fig.~\ref{fig:rolling_certificate_test_gap} shows the more realistic rolling-window certification results. The upper panel uses the daily-trajectory safety outcome based on the per-day maximum voltage deviation at $\nu=0.060$ p.u. The lower panel uses per-time-step voltage deviations at $\nu=0.060$ p.u., sampled every four 15-min control steps to reduce temporal correlation in the calibration dataset. Although the nominal CP upper bounds are generally informative, they are not always sufficient: in some rolling windows, the probability of unsafe operation can exceed the nominal certificate. This is because Assumption~\ref{ass:stated_calibration_distribution} no longer holds in the rolling-window cases. The adversarial CP upper bounds, computed with a $1.5\%$ PGD perturbation budget, enhance the robustness of the certificate in these windows and successfully cover the next-month unsafe probability for all rolling windows.

\begin{figure}[t]
\centering
\includegraphics[width=\linewidth]{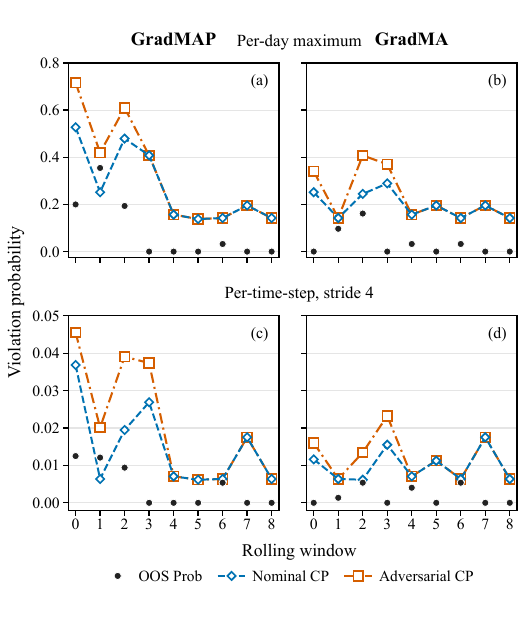}
\vspace{-13mm}
\caption{Case Study II rolling-window certification results at $\alpha=0.01$ for GradMAP in plots (a) and (c), and GradMA in plots (b) and (d). Plots (a) and (b) use the \emph{per-day maximum} voltage deviation at threshold $\nu=0.060$ p.u. as the network safety criterion. Plots (c) and (d) use \emph{per-time-step maximum} voltage deviations at $\nu=0.060$ p.u. as the network safety criterion. When constructing the calibration dataset, the per-time-step maximum voltage deviations are sampled every four 15-min control steps to reduce temporal correlation. Nominal CP refers to the proposed method in Section \ref{sec:finite_sample_safety_certificates}, while adversarial CP is the one with extra robustness with $\rho=1.5\%$ in Section \ref{sec:sample_space_robust_certification}. Calibration emp refers to the frequency of unsafe scenarios evaluated on the calibration set (the commonly used held-out method), and OOS prob refers to the true unsafe probability, which is estimated by evaluating on 10,000 out-of-sample scenarios.}
\label{fig:rolling_certificate_test_gap}
\end{figure}

\section{Discussion}\label{sec:discussion}

The validity of the proposed verification method relies on the key assumption that the calibration dataset is composed of independent samples drawn from the same distribution as the real-world data. A key challenge is that most real-world grid data, including demand and renewable generation data, are time series. This can make the collected data non-independent or cause the data distribution to shift, both of which are shown in our rolling-window case study. However, as discussed, certification in power grids (at least the first stage) is likely to be performed on a simulator because it is unlikely that an unverified AI system can be deployed in a real system. In this case, it may be possible to recover independence under Assumption~\ref{ass:stated_calibration_distribution} by using a simulated data generator such as HEDGE. There are also other approaches to reduce temporal dependence. One solution, as suggested in \cite{lindemann2025formal}, is to verify trajectory-level properties and collect independent trajectory samples whenever possible, which is the setting of our Case I. Another solution is to sample only a subset of the data to reduce temporal dependence, as in our Case II and more specifically Fig.~\ref{fig:rolling_certificate_test_gap}(c) and (d). 
Another issue associated with using simulators for verification is the sim-to-real gap. One can model this gap as another probability distribution inferred from experiments that compare the simulator and the real system. Then, in our verification framework, such a distribution can be used to generate samples absorbed into each of our calibration scenarios $W^j$ and modify the score calculation process for $M(\pi,W)$ in Eq.~\eqref{eq:threshold_safety_process}.




Also, to make our Assumption \ref{ass:stated_calibration_distribution} hold, it is noteworthy that calibration data should not be used to train, tune, select, or early-stop $\pi$, which ensures that the derived dataset of $L_\nu(\pi,W)$ has the same distribution as that under the future actual grid operating conditions. Therefore, if one wants to use the method as a learning metric for safe AI, they need to have another held-out dataset. The primary use of the proposed method is as a certification gate before deployment.


The sample requirement of our framework can be explicitly quantified to provide extra insights. If no unsafe outcomes are observed, we have \(U(0;n,\alpha)=1-\alpha^{1/n}\). Therefore, certifying \(p_\pi\leq\varepsilon\) with zero observed unsafe outcomes requires
\begin{equation}
    n\geq \left\lceil\frac{\log\alpha}{\log(1-\varepsilon)}\right\rceil
    =
    \left\lceil\frac{\log(1/\alpha)}{-\log(1-\varepsilon)}\right\rceil
    \overset{\varepsilon\to 0}{\approx} 
    \left\lceil\frac{\log(1/\alpha)}{\varepsilon}\right\rceil .
    \label{eq:zero_violation_samples}
\end{equation}
Thus, the number of calibration samples required to certify an extremely low unsafe probability grows at a rate of approximately \(1/\varepsilon\). For \(\alpha=0.05\), this gives approximately 299, 2,995, and 29,956 samples for \(\varepsilon=10^{-2}\), \(10^{-3}\), and \(10^{-4}\), respectively.

We can also quantify the number of calibration samples required for extremely high-confidence certification. For fixed \(\varepsilon\), Eq.~\eqref{eq:zero_violation_samples} shows that this number grows only logarithmically with \(1/\alpha\). For example, when \(\varepsilon=10^{-3}\), increasing the confidence level from \(95\%\) (\(\alpha=0.05\)) to \(99.9999\%\) (\(\alpha=10^{-6}\)) increases the required sample size from \(2,995\) to \(13,809\). Therefore, achieving higher confidence is comparatively easier than certifying an extremely low unsafe probability. This property is particularly advantageous in power system applications because grid operators may specify a high reliability requirement, and the above analysis informs that increasing the certification reliability is more sample-friendly.






\section{Conclusion}
\label{sec:conclusion}

This paper develops a finite-sample probabilistic safety certification framework for black-box multi-agent AI-based grid-edge coordination problems. The computational complexity remains of the same order as that of commonly used held-out evaluation, making the method scalable to large systems in the near- to mid-term while providing a finite-sample validity guarantee absent from standard held-out evaluation. The result relies on the assumption that the calibration dataset is i.i.d. with the distribution to be verified, which can be violated by distribution shift or simulator mismatch. To solve this, we therefore combine the nominal certificate with physically interpretable sample-space adversarial attacks, a concept widely used in the AI field to investigate model fragility. We perform case studies on grid-edge flexibility coordination with 1{,}000-agent AI models coupled through three-phase unbalanced AC load flow constraints, which verified the finite-sample safety guarantee and demonstrated the value of integrating adversarial attacks into a rolling-window training-certification-deployment workflow.

\bibliographystyle{IEEEtran}
\bibliography{ref}

\end{document}